\documentclass[12pt,a4paper]{article}
\usepackage[centertags]{amsmath}
\usepackage{amsfonts,amsthm,amssymb}
\usepackage{amssymb}
\usepackage{amsmath}
\usepackage{graphicx}
\usepackage{booktabs}
\usepackage{appendix}
\usepackage[hmargin=2.6cm,vmargin=2.6cm]{geometry}
\usepackage{etoolbox}
\usepackage{tikz}
\usepackage{soul}
\usepackage{comment}
\usepackage{float}
\usepackage{setspace}
\usepackage{threeparttable}
\usepackage{array}
\usepackage{booktabs}

\usepackage{blindtext}

\theoremstyle{definition}

\newcommand{\gor}{\rightarrow}

\newcommand{\p}{\mathbf{P}}

\newtheorem{lemma}{Lemma}

\newtheorem{theorem}{Theorem}

\newtheorem{corollary}{Corollary}

\newtheorem{definition}{Definition}

\newtheorem{exer            cise}[theorem]{Exercise}

\newtheorem*{lem:main}{Lemma \ref{lem:main}}

\newtheorem*{cor:int}{Corollary \ref{cor:int}}

\begin{document}

\title{Sequential Naive Learning\footnote{Itai acknowledges support from the Ministry of Science and Technology
(grant 19400214). Yakov acknowledges support from the Israel Science Foundation (grant 2021296). Manuel acknowledges the support from the Spanish Ministry of Science, Innovation, and Universities (ref. ECO2015-63711-P) andthe Department of the Economy and Knowledge of the Generalitat de Catalunya (ref. 2017 SGR 1244). }}
\author{Itai Arieli\thanks{%
Faculty of Industrial Engineering and Management, Technion: The Israel Institute of Technology,
iarieli@technion.ac.il}, Yakov Babichenko\thanks{%
Faculty of Industrial Engineering and Management, Technion: The Israel Institute of Technology,
yakovbab@technion.ac.il}, Manuel Mueller-Frank\thanks{%
IESE Business School, University of Navarra, mmuellerfrank@iese.edu}}
\maketitle
\begin{abstract}
We analyze boundedly rational updating from aggregate statistics in a model
with binary actions and binary states. Agents each take an irreversible action in sequence after observing the unordered set of previous actions. Each agent first forms her prior based on the aggregate statistic, then incorporates her signal with the prior based on Bayes rule, and finally applies a decision rule that assigns a (mixed) action to each belief. If priors are formed according to a discretized DeGroot rule, then actions converge to the state (in probability), i.e., \emph{asymptotic learning}, in any informative information structure if and only if the decision rule satisfies probability matching. This result generalizes to unspecified information settings where information structures differ across agents and agents know only the information structure generating their own signal. Also, the main result extends to the case of $n$ states and $n$ actions.
\end{abstract}

\section{Introduction}\label{introduction}
Many decisions under uncertainty require individuals to consider both their own private information and the earlier decisions of other individuals who faced the same decision problem. The inference based on the latter is called observational learning. Important examples of observational learning are social learning, where the earlier choices of friends or social contacts are observed, and aggregate-statistic based learning, where the available information consists of average reviews, sales statistics, and so on. The importance of social learning for ``real-world" decision making has been established for the purchase of consumer goods (see Feick and Price \cite{feick1987market}), agricultural practice (see Munshi \cite{munshi2004social} and Conley and Udry \cite{conley2010learning}), and microfinance (see Banerjee, Chandrasekhar, Duflo and Jackson \cite{banerjee2013diffusion}), among others. Similarly, there is ample empirical evidence for aggregate statistics impacting consumer choices;  examples include book best-seller lists (Bonabeau \cite{bonabeau2004perils}), software download counts (Hanson and Putler \cite{hanson1996hits}), online reviews (Huang and Chen \cite{huang2006herding}), and opening weekend movie ticket sales (Moretti \cite{moretti2011social}).

In the literature, two predominant approaches to the study of observational learning have emerged: Bayesian and boundedly rational. In the Bayesian approach agents make fully rational inferences on the private information of all agents based on the observed decisions.
Despite being valuable as a benchmark, the Bayesian approaches faces severe limitations. First, Hazla, Jadbabaie, Mossel and Rahimian \cite{hkazla2019reasoning} have shown Bayesian updating in social learning settings to be highly computationally intense. To clarify the complexity, note that in a social learning model the information each agent bases her action upon has two components: the past actions she observes and a private signal. The complexity arises in the inference from the observed past actions to the underlying information. This task is complicated by itself but even more so for aggregate statistic-based learning since the aggregate statistic provides no direct information about the observations each action was based upon. Second, Bayesian inference crucially relies upon knowing either the state-dependent distributions that generate the signals of each agent, or a distribution over the state-dependent distributions. Thus Bayesian updating is limited by agents' computational sophistication and more fundamentally by their knowledge of the environment 

Since there is empirical validation for aggregate statistics impacting real-world decisions, which due to the limitations described above are unlikely to be based on Bayesian updating\footnote{For experimental evidence on the failure of Bayesian updating in social learning environments see Kübler and Weizsäcker \cite{kubler2004limited}, Eyster, Rabin, and Weizsäcker \cite{eyster2015experiment}, March and Ziegelmeyer \cite{march2015even}, and Enke and Zimmermann \cite{enke2019correlation}.}, a formal study of this case is warranted. While the literature on Bayesian and boundedly rational social learning is extensive, to the best of our knowledge there is no thorough analysis of boundedly rational observational learning based on aggregate statistics. This paper can be seen as a first step toward developing the theoretical foundations and implications of aggregate statistics-based boundedly rational observational learning. 


The core research question we address in this paper is the following:
Do there exist aggregate statistics-based boundedly rational heuristics that enable long-run actions to optimally incorporate all privately held information, i.e. asymptotic learning? We address this question in a setting with conditionally independent and identically distributed signals as well as in an unspecified information structure setting where agents do not know how the signals of others are generated.

We borrow the standard herding model with binary states and binary actions employed in the social learning literature.\footnote{See, for example, Bikhchandani, Hirschleifer, and Welch \cite{bikhchandani1992theory}, Smith and Sorensen \cite{smith2000pathological}, Mossel, Sly and Tamuz \cite{mossel2015strategic}, and Bohren \cite{bohren2016informational}.}   
The binary herding model is both well suited to our analysis and objective since in many cases the observed aggregate statistics are binary. For example, adoption statistics and sales statistics are binary in the case of duopolies. However, in order to make our analysis more complete we generalize our main result to a general finite set of states and actions.   

The precise model we analyze can be described as follows. A countable infinite set of agents take an irreversible binary action  sequentially. Each agent receives a payoff of $1$ if her action matches the realized state and a payoff of $0$ otherwise. Prior to taking her action, each agent observes the unordered set of previous actions and a private signal that is drawn independently across agents conditional on the realized state. Each agent takes an action based on an updating heuristic in three steps. First, a \emph{prior formation heuristic} assigns a prior belief over the state space to the observed aggregate statistic. Second, a \emph{posterior formation heuristic} aggregates the prior with the private signal into a posterior. Third, a \emph{decision rule heuristic } maps the posterior of an agent to a mixed strategy.\footnote{In most of the existing literature there is no distinction between the prior formation, signal aggregation, and decision rule heuristics. In the DeGroot model the average of observed beliefs is the new belief. In quasi-Bayesian updating the selected action maximizes the expected utility maximizing action given the quasi-Bayesian belief. Arieli, Babichenko, and Mueller-Frank \cite{arieli2019naive} is an exception that allows for general decision rules, albeit in the repeated interaction setting.}
Each triple of a prior formation, a posterior formation, and a decision rule heuristic defines an updating heuristic. In this paper we focus on the following subclass of updating heuristics:

\paragraph{Prior formation} We analyze the asymptotic learning properties of the two most common prior formation functions studied in the literature in different settings: DeGroot prior formation and quasi-Bayesian prior formation.
   
Under a DeGroot updating heuristic \cite{degroot1974reaching}  continuous beliefs (or actions) are updated by averaging over the observed set of beliefs (or actions).\footnote{The literature on DeGroot updating is too vast to cite. A few notable examples are, DeMarzo, Vayanos, and Zwiebel \cite{demarzo2003persuasion} and Golub and Jackson \cite{golub2010naive,golub2012homophily}.} We adjust DeGroot updating to the binary action setting as follows: the DeGroot prior belief of, say, state $\omega=1$ is equal to the proportion of observed actions that equal $1$.\footnote{For the discretized version of DeGroot see Chandrasekhar, Larreguy and Xandri \cite{chandrasekhar2020testing}, Leshem and Scaglione \cite{leshem2018impact}, and Arieli, Babichenko and Mueller-Frank \cite{arieli2019naive}. } Thus DeGroot prior formation is invariant in the properties of the information structure.

Under quasi-Bayesian prior formation, which was first introduced by Eyster and Rabin \cite{eyster2010naive}\footnote{Eyster and Rabin \cite{eyster2010naive} refer to this as \emph{best response trailing naive inference.}} each observed action is treated as if it were a Bayesian best response to only the private signal of the agent taking the action.\footnote{Further studies of this heuristic include Mueller-Frank and Neri \cite{mueller2020general}, Levy and Razin \cite{levy2018information}, Dasaratha and He \cite{dasaratha2020network}, and Arieli, Babichenko, and Shlomov \cite{arieli2019robust}.} For experimental findings supporting the quasi-Bayesian updating heuristic see Eyster, Rabin, and Weizsacker \cite{eyster2015experiment}, Mueller-Frank and Neri \cite{mueller2020general} and Dasaratha and He \cite{dasaratha2019experiment}. Since agents observe only the unordered set of past actions, the identities of the agents taking a given action are not observed. This limits quasi-Bayesian updating to the conditionally i.i.d. signal case.\footnote{Unless agents have a distribution over the order with which actions are taken.}
    
\paragraph{Posterior formation} Here we focus on a single heuristic: the Bayesian one. That is, 
the signal is aggregated in a Bayesian manner for the given prior. Such a modeling choice is common in the literature on boundedly rational learning; see, e.g., Eyster and Rabin\cite{eyster2010naive}, Guarino and Jehiel \cite{guarino2013social}, Jadbabaie, Molavi, Sandroni and Tahbaz-Salehi\cite{jadbabaie2012non}, and Dasaratha and He \cite{dasaratha2020network}.

\paragraph{Decision rule} The decision rule assigns a mixed strategy to every posterior belief, as opposed to a deterministic strategy as is the norm in the literature. We impose no restrictions on the decision rule beyond some mild natural assumptions of monotonicity and semi-continuity.

We are interested in the information aggregation properties of the long-run actions when agents follow a common updating heuristic. We say that such a heuristic satisfies \emph{asymptotic learning} if the probability of agents selecting the optimal action converges to one along the sequence of agents.

Our main result, Theorem \ref{th:main}, shows that if the prior formation function is DeGroot, then asymptotic learning holds for \emph{every} (conditionally i.i.d.) information structure if and only if the decision rule equals probability matching. 
Surprisingly, Theorem \ref{th:main} finds that in the binary action setting combining the two most common prior and posterior formation heuristics, a probability matching decision rule achieves asymptotic learning in any information structure, and is the only decision rule to do so. This differs sharply from the asymptotic learning properties of Bayesian learning. Smith and Sorensen \cite{smith2008rational} show that asymptotic learning fails in the model considered here if private signals are bounded.\footnote{Private signals are bounded if the support of the posterior belief conditional on the signal does not contain 0 or 1.} Instead, the boundedly rational heuristic combining DeGroot prior formation with probability matching achieves asymptotic learning for bounded \emph{and} unbounded signal structures. Also, Theorem \ref{th:main} generalizes to the general case of $n$ states and $n$ actions, another distinction from the Bayesian learning case.

Next we consider the case of unspecified information structures. Signals are conditionally independent but not identically distributed across agents, and each agent knows only the information structure generating her own signal. For unspecified information structures, Theorem \ref{th:gl} shows that DeGroot prior formation paired with probability matching achieves asymptotic learning, if the signal distribution of each agent satisfies a weak informativeness property. 
Our results can be seen as a normative foundation for the DeGroot heuristic and the probability matching decision rule. Jointly they achieve learning in environments where Bayesian learning fails to do so, and even in environments where Bayesian learning is not applicable due to insufficient information about the information structure. 

Finally, Theorem \ref{th:eyster} shows that asymptotic learning fails under quasi-Bayesian prior formation, for every decision rule, in every (conditionally i.i.d.) information structure. Under quasi-Bayesian prior formation the prior formation is tailored to the information structure. By contrast DeGroot prior formation is invariant in the information structure. Nevertheless, DeGroot prior formation can achieve asymptotic learning in any information structure while quasi-Bayesian prior formation fails to achieve asymptotic learning in \emph{every} information structure.


The DeGroot belief-updating heuristic is arguably the most prominent in the literature and it is compelling due to its simplicity. Our results highlight that even in coarse action environments it enables long-run agents to approximately learn the realized state from an unordered set of binary actions (that are conditionally dependent) without any knowledge of the signal structures underlying them.\footnote{If all agents included in the group apply the DeGroot probability matching heuristic.} Thus DeGroot prior formation is an attractive proposition for agents, particularly in light of Theorem \ref{th:eyster}. Similarly, our results favor the use of probability matching as a decision rule in an observational learning environment due to its positive learning externality.  

The rest of the paper is organized as follows. Section \ref{literature} discusses the relation 
to the literature. Section \ref{model} introduces the model. Section \ref{section:main results} 
presents our asymptotic learning results for DeGroot and quasi-Bayesian prior formation. Section \ref{sec:extensions} establishes the asymptotic learning result for signals that are not identically distributed conditional on the state and discusses further extensions. Section \ref{sec:conclusion} concludes with a discussion and interpretation of our results and some thoughts on future work. All proofs are relegated to the appendix. 

\section{Related Literature}\label{literature}
We note first that previous work on boundedly rational heuristics that has provided positive information aggregation results considered rich action environments where actions are equal to beliefs or allow a one-to-one inference on the signal.\footnote{See Golub and Jackson \cite{golub2010naive}, Molavi, Tahbaz-Salehi and Jadbabaie \cite{molavi2018theory}, Jadbabaie, Molavi, Sandroni, and Tahbaz-Salehi \cite{jadbabaie2012non}, and Dasaratha and He \cite{dasaratha2020network}.} 
In the binary action case the mapping from  actions to beliefs or signals is considerably coarser than in the rich action case. Thus we analyze the binary action case to understand information aggregation properties of boundedly rational heuristics in environments least favorable for it.

The most related papers to ours are Eyster and Rabin \cite{eyster2010naive} and Dasaratha and He \cite{dasaratha2020network} who study the sequential setting. Both analyze the asymptotic learning properties of the quasi-Bayesian belief updating heuristic paired with the (Bayesian) deterministic decision rule that maximizes expected utility given the belief. Eyster and Rabin \cite{eyster2010naive} pioneered this modeling approach and show that asymptotic learning fails as herds on the incorrect action occur with positive probability. Dasaratha and He \cite{dasaratha2020network} analyze a rich action setting with Gaussian signals. They provide a characterization of the observational network structure that achieves asymptotic learning. Our paper generalizes both the information structure and the class of updating heuristics considered in these related papers. We complement their result by showing that in the binary action setting no decision rule can achieve asymptotic learning under quasi-Bayesian prior formation.

A series of recent papers analyzes the implications of Bayesian agents having mis-specified beliefs about the environment in a social learning environment. Bohren \cite{bohren2016informational} considers the sequential social learning model with binary actions where a proportion $p$ of agents observe the history of actions while others do not. She shows that if agents assume that the proportion is $p'\neq p$, asymptotic learning might fail. Bohren and Hauser \cite{bohren2020social} extend this approach in providing a general framework of misspecified beliefs and how they impact social learning outcomes. Frick, Iijima, and Iishi \cite{frick2020misinterpreting} provide conditions under which learning outcomes are not robust to minimal misspecifications.\footnote{See also Frick, Iijima, and Iishi \cite{frick2020stability}.} Our Theorem \ref{th:gl} extends their approach by considering a unspecified information environment, i.e., where agents have no knowledge of the information structure, but nevertheless achieve asymptotic learning through the DeGroot probability matching heuristic. 


Finally, our asymptotic learning result extends the scope of previous naive learning results from the rich action environment of Golub and Jackson \cite{golub2010naive} and Jadbabaie, Molavi, Sandroni, and Tahbaz-Salehi \cite{jadbabaie2012non} to a binary action environment.\footnote{Besides the cardinality of the action space, there are other considerable differences as both papers mentioned consider a model of repeated interaction.} In this we coincide with Arieli, Babichenko, and Mueller-Frank \cite{arieli2019naive} who analyze a generalized version of the DeGroot probability matching heuristic in the repeated interaction setting. In contrast to Arieli, Babichenko, and Mueller-Frank \cite{arieli2019naive} we provide a characterization of the decision rules that achieve asymptotic learning under DeGroot prior formation in the sequential herding model and show this decision rule to be probability matching. Further, we differ in considering the unspecified information environment and the result provided there. 

\section{The Model}\label{model}
Our setting is inspired by the standard herding model (see Banerjee \cite{banerjee1992simple}, Bikhchandani, Hirshleifer, and Welch \cite{bikhchandani1992theory}, and Smith and Sorensen \cite{smith2000pathological}). The model is described by the tuple  $(N,\Omega,A,\mu,(F,S))$, where $N$ is a countably infinite set of agents who decide sequentially on a binary action $a\in A=\{0,1\}$. The state of the world $\omega\in\Omega=\{0,1\}$ is drawn at $t=0$ according to the prior $\mu$.
The realized state is not observed by the agents.


Thereafter, at every time $t\geq 1$ agent $t$ takes an action $a_t$. As in the standard herding model, the utility of an agent is $1$ if her action equals the state and it is $0$ otherwise. Unlike the standard model, we assume that each agent observes ``only" the unordered set of actions that preceded her, as opposed to also the order in which they were taken. That is, the 
aggregate statistic observed by agent $t$ can be represented as a pair of numbers $(m,k)$ such that $m+k=t-1$ and where $m$ (respectively $k$) corresponds to the number of $0$ (respectively $1$) actions that are observed by agent $t$. We next describe in detail our assumptions on the information structure and the process according to which actions are taken.

\subsubsection*{The Information Structure}

The \emph{information
structure} of our model is described by the tuple $(\mu,F_{0},F_{1},S)$. We denote the common prior by $\mu$, where $\mu=\Pr \left[ \omega =1\right] $, $S$ is a measurable space of signals, and  $F_\omega\in\Delta(S)$ is a state-dependent distribution over the signals given the state $\omega=0,1$. We throughout assume that signals are informative, i.e., $F_0\neq F_1$.

In addition to observing the unordered history of actions $(m,k)$,
each agent $t$ receives a private signal $s_t\in S$ that is drawn independently from the signal of all other agents according to $F_\omega\in\Delta(S).$ Let $\p_{F,\mu}\in\Delta(\Omega\times S)$ denote the measure that is generated by $\mu$ and $F$. 


\subsubsection*{Belief Updating and Decision Rule}
We now introduce a general approach to the agents' deliberations that are finalized in their action choice. The class of heuristics we analyze decouples the belief formation process explicitly from the action choice while sharing properties with Bayesian updating.
Concretely, we decompose the procedure according to which an action is chosen into three steps: 1) prior formation given the history, 2) posterior formation incorporating the private signal, and 3) the action choice given the posterior belief. We now describe the three steps in detail.


\textbf{Prior formation:} The (boundedly rational) prior formation given the unordered history occurs according to a \emph{prior formation function}.
Recall that the history observed by agent $t$ can be represented as a pair of numbers $(m,k)$ such that $m+k=t-1$ and where $m$ (respectively $k$) corresponds to the number of $0$ (respectively $1$) actions that are observed by agent $t$. A prior formation function is a mapping $b:\mathbb{N}\times \mathbb{N}\rightarrow[0,1]$ that assigns a prior probability $b(m,k)$ to state $\omega=1$ given the number of observed choices of $0$ and of $1$.

\textbf{Posterior formation:} Each agent incorporates her private signal with the prior to form her posterior belief. We assume that the agent knows the signal-generating distributions $(F_{0},F_{1})$ and applies Bayes' rule to form her posterior belief given her signal, treating the prior as a true Bayesian prior. 
For every signal $s\in S$, let $p(s)=\p_{F,\frac 1 2}(\omega=1|s)$ be the conditional probability of state $\omega=1$, given the signal $s$ and the uniform prior $\mu=\frac 1 2$. Thus, for every $s\in S$ and a given prior $b(k,m)=b$, the agent's posterior probability of state 1 equals:
\begin{align}\label{eq:bayesian}
   \frac{p(s)b}{p(s)b+(1-p(s))(1-b)}. 
\end{align}



\textbf{Decision rule:} Each agent selects her action given her posterior belief. We assume that the action choice is determined by a \emph{decision rule}, i.e., an increasing function $\sigma:[0,1]\rightarrow[0,1]$ that assigns a mixed strategy to each belief. More precisely, for a posterior $y$ the term $\sigma(y)$ represents the probability that the agent plays action $a=1$, and $1-\sigma(y)$ is the probability that she plays action $a=0$. 
We restrict attention to decision rules 
$\sigma$ that are continuous at $0$ and $1$ and piecewise continuous elsewhere in the following sense: 
either it holds at any $x\in(0,1)$ that $\lim_{y\gor x^+}\sigma(y)=\sigma(x)$, or it holds that at any point $x$ $\lim_{y\gor x^-}\sigma(y)=\sigma(x)$. That is, $\sigma$ is either right-continuous at any point or left-continuous at any point. One decision rule that will play a prominent role in our results is \emph{probability matching} where $\sigma(x)=x.$ That is, action $a=1$ is played with a probability that coincides with the posterior probability of state $\omega=1$.

Let us briefly discuss our approach in the context of the existing literature. The use of a boundedly rational heuristic that assigns a prior belief to the observed history in combination with the incorporation of the private signal via Bayes' rule is quite common in the literature: for example, Eyster and Rabin \cite{eyster2010naive}, Guarino and Jehiel \cite{guarino2013social}, and Dasaratha and He \cite{dasaratha2020network} employ this approach in the strict sequential irreversible action setting, and Jadbabie et al. \cite{jadbabaie2012non} and Molavi, Tahbaz-Salehi and Jadbabaie \cite{molavi2018theory} in the repeated interaction setting. 
The main deviation from the boundedly rational literature lies in our formulation of the decision rule so that it allows for randomization. Existing papers overwhelmingly assume, either implicitly or explicitly, that agents best respond to their (boundedly rational) beliefs, resulting in deterministic decision rules.\footnote{The one exception is the companion paper Arieli, Babichenko, and Mueller-Frank \cite{arieli2019naive}, which analyzes the repeated interaction setting.}


We henceforth assume that the process starts at time $t=3$ and that the history at this time is $(m,k)=(1,1)$.  This assumption that both actions are present where the first agent takes her action is needed in order to avoid extreme priors of $0$ or $1$ in the DeGroot prior formation; both of these extreme priors cause the entire process to be stabilized on a single action ($0$ or $1$).

\subsubsection*{DeGroot and quasi-Bayesian prior formation}
As defined, our three-step approach captures a wide range of heuristics. In particular, our assumptions capture three commonly applied prior formation functions that have been extensively analyzed in the literature: Bayesian, DeGroot, and quasi-Bayesian updating (also called naive inference). We discuss each in turn to provide further insight into every step of our heuristic.\footnote{We include Bayesian prior formation in the discussion as a benchmark and to illustrate the properties of the other two prior formation heuristics.} 

So far we have only assumed that each agent knows the signal-generating functions $(F_{0},F_{1})$. To capture Bayesian and quasi-Bayesian updating, let us additionally assume that agents commonly know the information structure $(\mu,F_{0},F_{1},S)$. Under Bayesian updating, the decision rule of each agent assigns the action $1$ to every belief (weakly) above $\frac 1 2$ and the action $0$ to every belief (strictly) below $\frac 1 2$. A \emph{Bayesian prior formation} function then assigns a prior belief to each unordered set of actions via Bayes' rule, assuming that actions are chosen based on the Bayesian decision rule and that the state and signals are generated according to $(\mu,F_{0},F_{1})$.

A \emph{quasi-Bayesian prior} formation function treats each observed action as if it were a best response to only the private signal of the agent taking the action. In other words, each agent is treated as a Bayesian agent who observes no history of actions and bases her decision only on the prior $\mu$ and her signal $s$. The quasi-Bayesian decision rule coincides with the Bayesian one in assigning the action $1$ to every belief (weakly) above $\frac 1 2$ and the action $0$ to every belief (strictly) below $\frac 1 2$. We shall call this decision rule the \emph{Bayesian decision rule} going forward. Let $e$ denote the quasi-Bayesian prior formation function. Assuming a common prior $\mu=\frac 1 2$, for $t\geq 3$ we have 
\begin{align}\label{eq:qb}
e(k,m)=\frac{q^{k}r^{m}}{q^{k}r^{m}+(1-q)^{k}(1-r)^{m}},    
\end{align}
where $q=\p_{F,\frac{1}{2}}(\omega=1|p(s)\geq\frac{1}{2})$ and $r=\p_{F,\frac{1}{2}}(\omega=1|p(s)<\frac{1}{2})$. Thus $q$ represents the probability a Bayesian agent assigns to state $\omega=1$ conditional on observing $a=1$ for an agent who took her action in isolation based solely on her private signal. Similarly, $r$ represents the probability a Bayesian agent assigns to state $\omega=1$ conditional on observing the action $a=0$ taken by an isolated agent. Hence $e(k,m)$ coincides with the posterior of a Bayesian agent who observes $k$ actions $1$ and $m$ actions $0$ that were taken based on the private signal only.  

Finally, a \emph{DeGroot prior formation} function sets the prior equal to the proportion of $1$ actions in the observed unordered set. Let $d$ denote the DeGroot prior formation function. For an observed set $(m,k)$ we have $d(k,m)=\frac{k}{k+m}$. This approach is inspired by DeGroot \cite{degroot1974reaching}, where continuous actions are updated by taking the average of observed actions. For binary actions, the DeGroot prior formation function has been analyzed in the repeated interaction model by Chandrasekhar et al. \cite{chandrasekhar2020testing} and Arieli, Babichenko, and Mueller-Frank \cite{arieli2019naive}. The former paper combines the DeGroot prior formation function with the Bayesian decision rule, and the latter paper allows for mixed decision rules as in this paper.  

\section{Learning Dynamics under DeGroot Prior Formation}\label{section:main results}
We first focus on the case where the information structure generating the signals of each agent is commonly understood by all agents.\footnote{Following the convention in the literature we assume that signals are i.i.d. across agents conditional on the state.} Thus the use of boundedly rational heuristics is justified only by the complexity of inference rather than additionally by insufficient knowledge of agents for the application of Bayesian updating. 
We turn to our main question, which concerns the information aggregation properties of the class of heuristics we introduced above. In particular, we will analyze two different notions of learning. Let $\p_{F,\mu,b,\sigma}\in\Delta(\Omega\times A^\infty)$ be the probability measure over the product of the state space and the action process that is generated by $\mu$, $F$, and an updating heuristic $(b,\sigma)$.
\begin{definition}
For a given information structure $(\mu,F_{0},F_{1},S)$ we say that an updating heuristic $(b,\sigma)$ satisfies \emph{asymptotic learning} if $\p_{F,\mu,b,\sigma}(a_t=\omega)$ approaches one as $t$ goes to infinity.    
\end{definition}
The second notion of learning, \emph{adequate learning}, is stronger and defined as follows. 
\begin{definition}
An updating heuristic $(b,\sigma)$ satisfies \emph{adequate learning} if for every information structure $(\mu,F_{0},F_{1},S)$ where $F_{0}$ and $F_{1}$ are absolutely continuous with respect to each other, it holds that $\p_{F,\mu,b,\sigma}(a_t=\omega)$ approaches one as $t$ goes to infinity.    
\end{definition}

Asymptotic learning is the common information aggregation notion in the literature: actions are required to converge in probability to the true state, for a given information structure. By contrast, adequate learning requires actions to converge to the truth for \emph{every} informative information structure.\footnote{An information structure is called informative if  $F_{0}$ and $F_{1}$ are absolutely continuous with respect to each other.} 
Applying these two notions, we now study the information aggregation properties of the DeGroot and the quasi-Bayesian prior formation functions. We do so by imposing the respective prior formation function
but allow for any decision rule. 

\subsection{Adequate Learning for Specified Information Structures}
We now present the first main result of the paper, which provides a characterization of the decision rules that achieve adequate learning under DeGroot prior formation. 

\begin{theorem}\label{th:main}
 Let $d$ denote the DeGroot prior formation function. The updating heuristic $(d,\sigma)$ satisfies adequate learning if and only if $\sigma(x)=x.$
\end{theorem}

There are several features of this result that are worthwhile to discuss. First, a very simple updating heuristic
enables actions to converge to the true state. In contrast, Bayesian prior formation would be exceedingly complex. Each agent would be required to consider the unordered set each predecessor might have observed, her Bayesian inference based on the signal and the set of signals that in combination would lead to her action. In contrast, DeGroot virtual prior formation simply averages the actions in the unordered set. Second, the updating heuristic combining DeGroot averaging with probability matching achieves asymptotic learning for \emph{every} informative information structure. Instead, in the setting we consider, adequate learning would fail for Bayesian agents. Any information structure that features bounded signals\footnote{Signals are bounded if the support of the Bayesian posterior belief conditional on the signal contains neither 0 nor 1.} fails to achieve asymptotic learning. Third, probability matching is necessary and sufficient for adequate learning under DeGroot prior formation. Thus, Theorem \ref{th:main} gives a normative foundation for both Degroot averaging and probability matching in an observational learning environment. We discuss this in more detail in Section \ref{sec:conclusion}.

A puzzling feature of Theorem \ref{th:main} is that probability matching not only enables adequate learning but is in fact the \emph{only} strategy to do so. Clearly, adequate learning is a strong condition since it requires asymptotic learning to hold for \emph{every} information structure. A natural follow-up question would weaken this requirement. For example, given a subclass of information structures $\mathcal{F}$, one may ask what decision rules enable asymptotic learning for this subclass. Our proof sheds some light on this question, but we leave a full analysis for future research.

\subsubsection*{Proof Outline of Theorem \ref{th:main}}
The proof of Theorem \ref{th:main} relies on the analysis of generalized Polya urn models due to Hill, Lane, and Sudderth \cite{hill1980strong}. In these Polya urn models an initial composition of blue and red balls is given in an urn. In discrete time, at every stage $t$, one new ball is added to the urn without replacement. This additional ball is blue with probability $f(x_t)$ and red with probability $1-f(x_t)$, where $x_t$ is the proportion of blue balls in the urn at the beginning of state $t$ and $f:[0,1]\rightarrow[0,1]$ is a measurable \emph{urn function}. Note that conditional on state $\omega$, our action process is a generalized urn process, where the blue and red balls correspond to the $1$ and $0$ actions respectively, and where the proportion $x_t$ corresponds to the fraction of $1$ actions at the beginning of period $t$. 

For sufficiency, consider the process of actions conditional on state $\omega=1$. Let probability matching be the decision rule. Sufficiency then follows from two auxiliary results. We first show that the urn function $f$ corresponding to probability matching, which further depends on the information structure, has the property that given a virtual prior of $x$, the probability of action $a=1$ is strictly larger than $x$ for all $x\in(0,1)$. A direct corollary to the results of Hill, Lane, and Sudderth \cite{hill1980strong} establishes that (conditional on state $\omega=1$) the proportion of $1$ actions, and thus the virtual prior $x_t$, almost surely converges to $1$ as $t$ goes to infinity. Thus, in particular, $\p_{F,\mu}(a_t=1|\omega=1)$ approaches one.

For necessity of probability matching for adequate learning, we consider a decision rule $\sigma$ such that $\sigma(x)>x$ for some value $x\in [0,1]$ (the case $\sigma(x)<x$ may be treated symmetrically). We establish the result by providing one information structure for which asymptotic learning fails. Let $\tilde{F}$ be a binary information structure, i.e., with two signals, which is sufficiently uninformative. Consider the  actions process conditional on state $\omega=0$. We show that the urn function $\tilde f$ corresponding to $\tilde \sigma$ and $\tilde F$ satisfies $\tilde{f}(x)>x$ for some $x$. By standard coupling arguments this implies that the long-run fraction of $0$ actions is bounded away from $1$ with positive probability. This in turn implies that $\p_{F,\mu}(a_t=0|\omega=0)$ does not converge to one and thus asymptotic learning fails.

\subsection{Learning Dynamics in the Unspecified Information Model}
We next relax the assumption of conditionally i.i.d. signals, and analyze adequate learning for conditionally independent but not identically distributed signals. One critical assumption common to Bayesian observational learning models is that all agents commonly understand how the signals of all agents are generated. Thus they share a common specified model of the information structure. A recently emerging literature considers the case where Bayesian agents share a common yet incorrect model of the information structure.\footnote{See the discussion in the section on related literature.} We go one step further and consider the case of an unspecified model of information. Agents know how their own signals are generated but have \emph{no} knowledge of the signal generating process of others, including that of no distribution over possible information structures. This prevents the use of both Bayesian and quasi-Bayesian prior formation.  
The  model is as follows. The state of the world $\omega$ is drawn according to prior $\mu\in(0,1)$, and each agent $t$ receives a private signal that is drawn by an information structure $F^t$ independently across agents. Note that since agents observe only the aggregate statistic of past actions, even knowing the joint information structure $\{F^t\}_{t\in\mathbb{N}}$ would not allow them to distinguish more informative from less informative actions, adding further complexity to the Bayesian-updating task. We impose one mild condition on the set of information structures ${F^t}_{t\in\mathbb{N}}$. We let $P^t\in\Delta([0,1])$ be the probability measure that is generated by $F_t$ conditional on the prior $\mu=\frac{1}{2}.$ Formally, recall that $p(s)=\p_{F,\frac{1}{2}}(\omega=1|s)$ is the conditional probability of $\omega=1$ given that the prior is $\mu=\frac 1 2$ and the realized signal is $s$.
For every Borel measurable subset $B\subseteq[0,1]$ we let
$$P_t(B)=\p_{F,\frac{1}{2}}(p(s)\in B).$$
Let $v_t$ be the variance of $P_t$. We say that the sequence $\{F_t\}_t$ is \emph{uniformly informative} if 
$\inf_t v_t>0.$

We provide the following extension of Theorem \ref{th:main}.
\begin{theorem}\label{th:gl}
Let $d$ denote the DeGroot prior formation function. If $\sigma(x)=x$ then $(d,\sigma)$ satisfies adequate learning within the class of uniformly informative information structures. 
\end{theorem}
Theorem \ref{th:gl} shows that the DeGroot probability matching heuristic achieves perfect information aggregation even in settings where agents have no knowledge of the order in which past actions were chosen, no knowledge of the information structure of others, and where information structures across agents indeed differ. It is worthwhile to discuss this result's distinguishing features from existing results.
Molavi, Tahbaz-Salehi, and Jadbabaie \cite{molavi2018theory} and Jadbabaie, Molavi, Sandroni, and Tahbaz-Salehi \cite{jadbabaie2012non} analyze a setting where agents receive private signals in each period. They show that a version of DeGroot prior formation paired with a Bayesian decision rule can lead to perfect information aggregation when signals are not identically distributed across agents. However, they do so in a setting where agents communicate their belief, and thus allow for much finer inference on the underlying signal of the agent than in our setting of binary actions. 
Recent work by Frick, Iijima, and Iishi \cite{frick2020misinterpreting} has highlighted the fragility of the asymptotic learning outcomes under Bayesian learning to minor errors in the agents' common understanding of the environment.\footnote{See also Acemoglu, Ozdaglar, and ParandehGheibi \cite{acemoglu2010spread} and Mueller-Frank \cite{mueller2018manipulating} for the fragility of DeGroot updating in the rich action setting.} By contrast, Theorem \ref{th:gl} provides an asymptotic learning result for an updating heuristic that is robust within an open set of (conditionally independent) information structures. 

We make two remarks on the uniform informativeness condition. First, one can show that without it Theorem \ref{th:gl} can fail. Otherwise one could construct a sequence of information structures where the informativeness of the signals is vanishing along the sequence and asymptotic learning fails. The failure can occur even if the condition $\lim\sup_t v_t>0$ is satisfied. Second, here we chose to express informativeness in terms of the variance of private beliefs. Other notions of informativeness will lead to the same result.
\subsubsection*{Proof Outline of Theorem \ref{th:gl}}
The proof is based on the following line of argument. Since the joint information structure is uniformly informative, we can find an (informative) information structure $F$ that is dominated (in the Blackwell-ordering sense) by all information structures $F_t$.
We then show that, conditional on state $\omega=1$ and a proportion $x_t$, the probability of agent $t$ choosing action $a_t=1$ under the information structure $F_t$ is larger than the probability of her choosing the same action under the information structure $F$. Theorem \ref{th:gl} then follows from Theorem \ref{th:main}.

\section{Learning under Quasi-Bayesian Prior Formation}
We next turn to quasi-Bayesian prior formation. Eyster and Rabin \cite{eyster2010naive} consider the sequential social learning model and show that Quasi-Bayesian prior  and Bayesian posterior formation paired with a Bayesian decision rule leads to failure of asymptotic learning (in a setting with bounded signals). 
As we show next, failure of asymptotic learning extends to every decision rule, in every information structure. Recall that under quasi-Bayesian updating, every observed action is treated as if it were based only on the private signal. To assure that every observed unordered set of actions is consistent with its interpretation, we assume that the prior belief is equal to half\footnote{Under the uniform prior a Bayesian agent, who takes his action based on his signal alone, selects either action with positive probability, for any pair of state-dependent signal-generating measures $(F_{0},F_1)$.}, $\mu=\frac{1} {2}$. Here we further assume that $F_0$ and $F_1$ are mutually absolutely continuous with respect to each other. In making this assumption we avoid cases where some signals fully reveal the realized state. 

Recall that the quasi-Bayesian prior formation function $e(k,m)$ is given by 
\begin{align}\label{eq:qb}
e(k,m)=\frac{q^{k}r^{m}}{q^{k}r^{m}+(1-q)^{k}(1-r)^{m}}    
\end{align}
where $q=\p_{F,\frac{1}{2}}(\omega=1|p(s)\geq\frac{1}{2})$ and $r=\p_{F,\frac{1}{2}}(\omega=1|p(s)<\frac{1}{2})$.


The following result provides a sharply negative result regarding the information aggregation properties of Quasi-Bayesian prior formation.
\begin{theorem}\label{th:eyster}
 Let $e$ denote the Quasi-Bayesian virtual prior formation function. The updating heuristic $(e,\sigma)$ fails to achieve asymptotic learning for every decision rule $\sigma$, in every informative information structure $(\mu=\frac{1}{2},F,S)$.
\end{theorem}
This result raises the question as to what causes the distinction between the sharp negative result of Theorem \ref{th:eyster} and the positive result of Theorem \ref{th:main}. The answer, we believe, relies on the fact that under the quasi-Bayesian heuristic the prior of agent $t$ is determined by the difference between the number of occurrences of actions $1$ and $0$. Thus the prior can be extreme and close to zero (or one) even when the proportion of actions equal to $1$ (or $0$) is close to $\frac{1}{2}$. This entails that, compared with the DeGroot prior, the quasi-Bayesian prior is more sensitive and goes to the extremes much faster. In particular, incorrect herds are harder to overturn under quasi-Bayesian prior formation.


\subsubsection*{Proof Outline of Theorem \ref{th:eyster}}
We next provide a proof outline for Theorem \ref{th:eyster}. Fix an information structure $F$ and, by way of contradiction, assume that asymptotic learning holds for some decision rule $\sigma$. We first show that conditional on state $0$, the probability of selecting action $1$ goes to one as the observed fraction of $1$ actions goes to one. 
We use this fact to couple the quasi-Bayesian action process, conditional on state $\omega=0$ with a random walk over the integers. The idea is that the coupled random walk moves to the right by one integer if $a_t=1$ and otherwise moves to the left. 

We then show, using the classic \emph{gambler's ruin probability formula}, that if the starting time $t_0$ of the random walk  corresponds to a history where the advantage of actions $1$ over $0$ is sufficiently large, then the random walk converges to infinity without ever crossing state $0$ with positive probability. Using our identification, this implies that $a_t$ converges to $1$ with positive probability conditional on state $\omega=0$, in contradiction to the fact that asymptotic learning holds.

\section{Extensions}\label{sec:extensions}
\subsubsection*{General Finite State Space}
Following the norm in the social learning literature, we have focused up to now on the binary model. A natural question to ask is whether Theorem \ref{th:main} carries forward to the case of a general finite state space $\Omega=\{\omega_1,\ldots,\omega_n\}$. We address this now.

Let the information structure be given by $(\mu,(F_\omega)_{\omega\in\Omega},S)$, where $\mu\in\Delta(\Omega)$ is the prior probability measure, $S$ is a measurable space, for every $\omega\in\Omega$ we have $F_\omega\in\Delta(S)$, and $F_\omega\neq F_{\omega'}$ for $\omega\neq\omega'$.
The action set is $A=\{a_1,\ldots,a_n\}$, where for each $1\leq k\leq n$ action $a_k$ yields a utility of $1$ in state $\omega_k$, and a utility of $0$ in any other state. As before, a decision rule is a mapping $\sigma:\Delta(\Omega)\rightarrow\Delta(A).$ We can replace the right- or left-continuity requirement with the requirement that for every $p\in \Delta(\Omega)$  it holds that $\lim_{q\rightarrow p}\sigma_k(q)=\lim_{q\rightarrow p}\sigma_k(q)$ for vectors $q$ such that $q_k<p_k$ (or vice versa).  

Agent $t$ now forms a prior probability $p^t\in\Delta(\Omega)$ according to the proportion choice of actions in the unordered set. She then incorporates her private signal in accordance with Bayes' rule. We claim that Theorem \ref{th:main} extends to this setting. To see why, let the realized state be $\omega_k\in\Omega$ and consider the following reduction to the binary case: at time $t$ we add a ``blue ball'' to the urn only if agent $t$ plays action $a_t$, and otherwise we add a ``red ball''. By Theorem \ref{th:main}, this reduction demonstrates that when the decision rule $\sigma$ equals probability matching, then the proportion of $a_k$ actions approaches one as time goes to infinity. The converse direction also follows immediately from Theorem \ref{th:main}.

\subsubsection*{General Updating Heuristics}


Without imposing any restrictions on the prior formation function and on the posterior formation function one may ask: which triplets $(b,a,\sigma)$ satisfy adequate learning? Here $b:\mathbb{N}^2 \to [0,1]$ is a general prior formation function and $a:[0,1]\times S \to [0,1]$ is a general function that generates a posterior belief based on the prior and the signal $s\in S$. Unfortunately, we are not yet able to fully answer this question and we leave it open for future research. 
Nevertheless, a few observations about this more general question are worth mentioning. 


Consider the case where the
prior formation function $b:\mathbb{N}\times\mathbb{N}\rightarrow[0,1]$ is \emph{size invariant}. That is, there exists an increasing continuous function $g:[0,1]\rightarrow[0,1]$ such that for every $(k,m)$, 
$$b(k,m)=g(\frac{k}{k+m})=g(d(k,m)).$$
Under size invariant prior formation the prior is determined by a general function of the proportion of $1$ actions, not necessarily the identity function. 

One particular case of the question above, would assume size invariant prior formation and Bayesian posterior formation. The objective then would be to characterize the class $(g,\sigma)$ of size invariant prior formation functions and decision rules for whom adequate learning holds. 
Let $x$ denote the proportion of $1$ actions and $y$ the posterior belief. Consider the case where for every $x$ and $y$ we have
$$g(x)=\frac{xc}{cx+(1-c)(1-x)}\text{ and } \sigma(y)=\frac{(1-c)y}{(1-c)y+c(1-y)}.$$
We call such a pair of functions a \emph{complementary pair}.
In words, the function $g$ maps a proportion $x$ to a prior $g(x)$ by treating $x$ as a Bayesian prior and updating it via Bayes rule with an imaginary signal $s$ such that $p(s)=c.$ The strategy $\sigma$ on the other hand takes the posterior $y$ and ``corrects'' the bias imposed on the proportion $x$ by updating $y$ via Bayes' rule with the imaginary countersignal $s'$ such that $p(s')=1-c$.
The following is a simple corollary of the proof of Theorem \ref{th:main}. 
\begin{corollary}
If $(g,\sigma)$ is a \emph{complementary pair}, then adequate learning holds.
\end{corollary}
To see this, note that such a complementary pair $(g,\sigma)$ is \emph{observationally equivalent} to the case where $g=id$ and $\sigma=id$ (as considered in Theorem \ref{th:main}) in the following sense: for every proportion $x$, state $\omega$, and information structure $F$, the probability that an agent who uses $(g,\sigma)$ will choose action $1$ is the same as $(id,id)$. Thus, conditional on any state $\omega$ the action process that is induced by 
$(g,\sigma)$ and by $(id,id)$ have exactly the same distribution. The corollary readily follows from this fact. 
One can further show that only complementary pairs are observationally equivalent\footnote{This claim is by no means straightforward.} to $(id,id)$. Based on the corollary, one may conjecture that observational equivalence with DeGroot prior formation paired with probability matching is a necessary condition for adequate learning within the class of size-invariant prior formation functions.

\section{Concluding Remarks}\label{sec:conclusion}
Our results provide a normative foundation for the DeGroot probability-matching heuristic in a observational learning environment with binary actions. We now briefly discuss several aspects of our contribution and mention some open questions for future work. First, for models of repeated interaction on social networks DeGroot updating has been shown to achieve learning in a setting where the action space is rich; see Golub and Jackson \cite{golub2010naive}. We show that in the sequential model where observations are limited to an aggregate statistic, DeGroot prior formation can achieve asymptotic learning even for binary actions in \emph{every} informative information structure, but only if the decision rule satisfies probability matching. For some information structures where probability matching achieves asymptotic learning, pairing the DeGroot prior formation with the (deterministic) Bayesian decision rule induces long-run actions that match the state only with probability half. Thus probability matching can dramatically improve long-run outcomes relative to the Bayesian decision rule. Despite probability matching violating expected utility maximization, there is a vast experimental literature documenting it in laboratory settings. The reasons behind the widespread use of probability matching have long puzzled researchers. In the words of Kenneth
Arrow \cite{arrow1958utilities}:

``\textit{We have here an experimental situation which is essentially of an
economic nature in the sense of seeking to achieve a maximum of expected
reward, and yet the individual does not in fact, at any point, even in a
limit, reach the optimal behavior. I suggest that this result points out
strongly the importance of learning theory, not only in the greater
understanding of the dynamics of economic behavior, but even in suggesting
that equilibria may be different from those that we have predicted in our
usual theory.}''

Our results point to one possible evolutionary justification for probability matching: it can achieve optimal asymptotic outcomes in observational learning environments where expected utility maximization fails. More precisely, probability matching induces a positive learning externality. The existence of what is seen as a violation of rationality might instead point to an intrinsic utility that weights outcomes of others. Our characterization leaves no room for any mixed learning heuristic other than \emph{exact} probability matching. However, we believe that for a fixed information structure, under the appropriate notion of proximity, any decision rule that is close to probability matching also achieves asymptotic learning.


Second, we derive the positive results in a model where agents observe only the unordered set of past actions. Note that the exact same results would hold if instead they were to observe the ordered set of past actions as is the norm in the sequential social learning literature.

Finally, our analysis leaves room for further work. We established probability matching as necessary and sufficient for adequate learning if prior formation is according to DeGroot. A more universal approach would instead allow for general prior formation and aim at characterizing the properties of the pair of the prior formation function and decision rule that achieve adequate learning. The size-invariant prior formation function introduced above provides one avenue of approach to this problem. We leave this for future work.

\section{Proofs}
We start with the proof of Theorem
\ref{th:main}. The proof of the theorem relies on \cite{hill1980strong} which is a classic paper on \emph{generalized urn processes}. A generalized urn process $\{x_t\}_{t\in\mathbb{N}}$ comprises an initial urn composition of $m>0$ blue balls and $k>0$ red balls and a measurable function $f:[0,1]\rightarrow[0,1]$. At any point in time $t\geq 1$, if the blue balls' proportion is $x_t$, then a new blue ball is added with probability $f(x_t)$ and a new red ball is added with probability $1-f(x_t)$. We first note that by Theorem 2.1 in \cite{hill1980strong} the process $\{x_t\}_{t\geq 0}$ converges almost surely to a limit.
The following two auxiliary lemmas follow directly from the main results in \cite{hill1980strong}.
\begin{lemma}\label{proposition:aux}
Consider an urn process and assume that $f$ is a continuous function.
If $f(x)>x$ ($f(x)<x$) for every $x\in(0,1)$, then $\lim_{t\rightarrow\infty} x_t=1$ ($\lim_{t\rightarrow\infty} x_t=0$). 
\end{lemma}

\begin{lemma}\label{proposition:aux2}
Consider an urn processes and assume that $f$ is a continuous function.
If there exists $x'\in (0,1)$ such that $f(x)>x$ in a left neighborhood of $x'$ and $f(x)<x$  in a right neighborhood of $x'$, then there exist an open interval $I$ where $x'\in I$ and a constant $n$ such that for $x_0\in I$ and $m+k\geq n$ it holds that $\lim_{t\rightarrow\infty}x_t=x'$ with positive probability.      
\end{lemma}
Moreover we make use of the obvious fact that if $g:[0,1]\rightarrow[0,1]$ is a piecewise constant such that $f\geq g$ and $\{y_t\}_t$ is the corresponding process with the same initial composition as $x_t$, then we can couple the two processes together such that $x_t\geq y_t$ for every $t$.

We start with a Lemma that is based on Lemma \ref{proposition:aux2}.
\begin{lemma}\label{lem:conv}
Consider an urn process and assume that the function $f$ is picewise continuous as defined above.
If $f(\tilde x)>\tilde x$ (or $f(\tilde x)<\tilde x$) for some $\tilde x\in(0,1)$, then there exist an interval $J\subseteq (0,1)$ and some $n>0$ so that for every initial condidions $(m,k)$ with $x_0\in J$ and $m+k>n$ it holds with positive probability that $\lim_{t\gor\infty} x_t\neq 0$ ($\lim_{t\gor\infty} x_t\neq 1$).
\end{lemma}
\begin{proof}[\textbf{Proof of Lemma \ref{lem:conv}}]
It follows from the facts that $f$ is either left- or right-continuous that there exists a closed interval $I=[a,b]\subset (0,1)$ such that $x\in I$ and $f(x)>x+\delta$ for every $x\in I$. Therefore, we can find a continuous function $g:[0,1]\rightarrow[0,1]$ and a point $x'\in I$ such that $f(y)\geq g(y)$ for every $y\in [0,1]$, and $g(y)>x$ in a left neighborhood of $x'$  and $g(y)<y$ on a right neighborhood of $x'$. It follows from Lemma \ref{proposition:aux2} that there exist an interval $J$ of $x'$ and a constant $n$ such that for every initial condition $(m,k)$ such that $x_0\in J$ the process $\{y_t\}_t$ converges to $x'$ with positive probability. Since for the corresponding process $\{x_t\}$  we have $f\geq g$, from the above comment we can couple the processes together such that $x_t\geq y_t$ for every $t$. Therefore if $m+k>n$ and $x_0\in J$, then  $\lim_{t\gor\infty} x_t\geq x'>0$ with positive probability.   
\end{proof}

\begin{proof}[\textbf{Proof of Theorem \ref{th:main}}]
We note that if the prior is $\mu$, then it follows directly from Bayes rule that the conditional probability of $\omega=1$ given a signal $s\in S$ is
$$p_\mu(s)=\p_{F,\mu}(\omega=1|s)=\frac{p(s)\mu}{p(s)\mu+(1-p(s))(1-\mu)}.$$
It follows from standard martingale considerations that
$$E_{F,\mu}[p_\mu(s)]=\mu.$$

Consider a prior $\mu$ and let $G^\mu$ be the posterior belief distribution. That is, for every measurable $A\subseteq[0,1],$ $$G_\mu(A)=\p_{F,\mu}(p_\mu(s)\in A).$$ Similarly, for every state $\omega$ we let  
$$G_\mu^\omega(A)=\p_{F,\mu}(p_\mu(s)\in A|\omega)$$
be the posterior distribution conditional on state $\omega$. Let $\underline \beta_\mu, \overline \beta_\mu\in[0,1]$ be the lower point and the upper point of the support of $p(s)$. 
Since $\frac{dG_\mu^1}{dG_\mu}(x)=\frac{x}{\mu}$ and $\frac{dG_\mu^0}{dG_\mu}(x)=\frac{1-x}{1-\mu}$ at every point in the support of $p_\mu(s)$ (see, e.g., Lemma 1 in \cite{arieli2020identifiable}), it follows that $G_1([0,x])<G_0([0,x])$ for every $x\in \underline (\beta_\mu, \overline \beta_\mu)$. Thus $G^\mu_1$ first-order stochastically dominates $G^\mu_0$.
We further note that, by definition, $E_{F,\mu}[p_\mu(s)]=\int_{[0,1]} x dG_\mu(x)$, and $E_{F,\mu}[p_\mu(s)|\omega]=\int_{[0,1]} x dG^\omega_\mu(x)$.
By the law of iterated expectation, 
\begin{align*}
&\mu=E_{F,\mu}[p_\mu(s)]=\mu E_{F,\mu}[p_\mu(s)|\omega=1]+(1-\mu)E_{F,\mu}[p_\mu(s)|\omega=0].
\end{align*}
Since $E_{F,\mu}[p_\mu(s)|\omega=1]>E_{F,\mu}[p_\mu(s)|\omega=0]$ we must have that
$$E_{F,\mu}[p_\mu(s)|\omega=1]>\mu>E_{F,\mu}[p_\mu(s)|\omega=0].$$

We note that for the strategy $\sigma$ the learning process conditional on any state $\omega\in\Omega$ is an urn process. In particular, in the case $\sigma(x)=x$, for every proportion of $\frac{m}{m+k}=x_t$ agents who play action $1$, agent $t$ plays action $a_t=1$ with probability 
$f(x_t)=E_{F,x_t}[p_{x_t}(s)|\omega]$, which is a continuous function of $x_t$. Thus, in particular, conditional on state $\omega=1$,  the probability that agent $t$ plays action $1$ is $E_{F,x_t}[p_{x_t}(s)|\omega=1]>x_t$. We further note that $E_{F,x_t}[p_{x_t}(s)|\omega=1]$ is a continuous function of $x_t$ since $p_{x_t}(s)$ is a continuous function of $x_t$ for any signal $s$. Thus it follows from the auxiliary Lemma \ref{proposition:aux} that the proportion $x_t$ converges to $1$ conditional on state $\omega=1$. Similarly, since $E_{F,x_t}[p_{x_t}(s)|\omega=0]<x_t$ the process $x_t$ converges to $0$ conditional on state $\omega=0$. Note that $f(x)=E_{F,\mu}[p_{x}(s)|\omega]$ approaches $0$ as $x$ approaches $0$ and $f(x)=E_{F,\mu}[p_{x}(s)|\omega]$ approaches $1$ as $x$ approaches $1$ for every state $\omega$. Therefore since $\lim_{t\gor\infty}x_t=\omega$ with probability one, we have that the probability of $a_t=\omega$ approaches one as $t$ goes to infinity. Hence asymptotic learning holds.         

We next show the converse direction. We need to show that if $\sigma\neq id$, then for any prior $\mu\in (0,1)$ there exists a state $\omega$ such that conditional on state $\omega$ it holds that $\lim_{t\gor\infty}x_t\neq \omega$ with positive probability. 

Since $\sigma\neq id$ it holds that $\sigma(\tilde x)\neq \tilde x$ for some $\tilde x\in(0,1)$. Therefore, either $\sigma(\tilde x)>\tilde x$ or $\sigma(\tilde x)<\tilde x$. Without loss of generality consider the first case. Since $\sigma$ is either left- or right-continuous at $\tilde x$, we can find an interval $I\subset (0,1)$ and a constant $\eta$ such that $\sigma(x)>x+\delta$ for every $x\in I$.

We recall that conditional on state $\omega$ action $1$'s proportion  $\{x_t\}_{t\in\mathbb{N}}$ is a generalized urn process where the probability of playing action $1$ is $f(x)=E_{F,\mu}[\sigma(p_x(s))|\omega=0]$. To see this, note that since for every $x$ the function $p_x(s)$ is continuous and increasing in $x$, it follows that if $\sigma$ is right-continuous at any point so is $f$ (and similarly with respect to left continuity). 

We claim that we can find an information structure $F$ and a point $x\in(0,1)$ such that 
$E_{F,\mu}[\sigma(p_x(s))|\omega=0]>x$. 

To see this, consider a binary information structure $F$ with $S=\{s_0,s_1\}$ and $p(s_\omega|\omega)=\frac{1}{2}+\epsilon$ for every $\omega$. It follows from Bayes' rule that $p_x(s_1)=\frac{(\frac{1}{2}+\epsilon)x}{(\frac{1}{2}+\epsilon)x+(\frac{1}{2}-\epsilon)(1-x)}$ and $p_x(s_0)=\frac{(\frac{1}{2}-\epsilon)x}{(\frac{1}{2}-\epsilon)x+(\frac{1}{2}+\epsilon)(1-x)}$.
Moreover, conditional on state $\omega=0$, the variable $p_x$ equals $p_x(s_1)$ with probability $\frac{1}{2}-\epsilon$ and $p_x(s_0)$ with probability $\frac{1}{2}+\epsilon$. 

We note that  $\lim_{\epsilon\gor 0} p_x(s_\omega)=x$. Therefore we can find $\eta>0$ and $x\in I$ such that $\sigma(p_x(s_\omega))>x$ for $\omega=0,1$. This implies in particular that $f(x)=E_{F,x}[\sigma(p_x(s))|\omega=0]>x.$ 

It therefore follows from Lemma \ref{lem:conv} that there exist an open interval $J\subset (0,1)$ and $n$ such that if $m+k>n$ and $x_0\in J$, then $\{x_t\}$ does not converge to $0$ with positive probability. 

First we may assume that, conditional on state $\omega=1$, it holds that $\lim_{t\rightarrow\infty} x_t=1$ with positive probability. Otherwise $(d,\sigma)$ does not satisfy asymptotic learning. Therefore, conditional on state $\omega=0$ for every $n'$ and $\epsilon>0$ there exists a time $t$  such that $x_t\geq 1-\epsilon$ and $t\geq n'$. This follows from the fact that any finite sequence of actions that has a positive probability conditional on state $\omega=1$. This holds true since $F_0$ and $F_1$ are absolutely continuous with respect to each other.

We note that for small enough $\epsilon$ the value $n'$ must be arbitrarily large. Hence, for small enough $\epsilon$, if $x_t$ reaches a point $x_t\geq 1-\epsilon$, then if $x_{t}$ crosses the interval $J$ from right to left we must have a time $t'$ such that $x_t\in J$ and $t>n$. In such a case we have that $\{x_t\}$ does not converge to $0$ with positive probability. Thus either the process crosses $x_t\geq 1-\epsilon$ and after that does not cross $J$ with positive probability or $x_t\in J$ and $t>n$. In any case $\{x_t\}$ does not converge to $0$ with positive probability. Hence asymptotic learning fails.


Without loss of generality we consider the first case. We claim that we can find an information structure $F$ and an interval $J$ such that 
$E_{F,\mu}[\sigma(p_\mu(s))|\omega=0]>\mu+\eta$ for every $\mu\in J$ and some $\eta>0$. 

To see this, consider a binary information structure $F$ with $S=\{s_0,s_1\}$ and $p(s_\omega|\omega)=\frac{1}{2}+\epsilon$ for every $\omega$. It follows from Bayes' rule that $p_x(s_1)=\frac{(\frac{1}{2}+\epsilon)x}{(\frac{1}{2}+\epsilon)x+(\frac{1}{2}-\epsilon)(1-x)}$ and $p_x(s_0)=\frac{(\frac{1}{2}-\epsilon)x}{(\frac{1}{2}-\epsilon)x+(\frac{1}{2}+\epsilon)(1-x)}$.
Moreover, conditional on state $\omega=0$, the variable $p_x$ equals $p_x(s_1)$ with probability $\frac{1}{2}-\epsilon$ and $p_x(s_0)$ with probability $\frac{1}{2}+\epsilon$. 

We note that  $\lim_{\epsilon\gor 0} p_x(s_\omega)=x$. 
Therefore we can find $\eta>0$ and $x\in I$ such that $\sigma(p_x(s_\omega))>x$ for $\omega=0,1$. This implies in particular that $f(x)=E_{F,x}[\sigma(p_x(s))|\omega=0]>x.$ 

It now follows from Lemma \ref{lem:conv} that there exist an open interval $J\subset (0,1)$ and $n$ such that if $m+k>n$ and $x_0\in J$, then $\{x_t\}$ does not converge to $0$ with positive probability. 

First we may assume that, conditional on state $\omega=1$, it holds that $\lim_{t\rightarrow\infty} x_t=1$ with positive probability. Otherwise $(d,\sigma)$ does not satisfy asymptotic learning. Therefore, conditional on state $\omega=0$ for every $n'$ and $\epsilon>0$ there exists a time $t$  such that $x_t\geq 1-\epsilon$ and $t\geq n'$. This follows from the fact that any finite sequence of actions that has a positive probability conditional on state $\omega=1$ also has a positive probability conditional on state $\omega=0$. This holds true since $F_0$ and $F_1$ are absolutely continuous with respect to each other.

We note that for small enough $\epsilon$ the value $n'$ must be arbitrarily large. Hence, for small enough $\epsilon$, if $x_t$ reaches a point $x_t\geq 1-\epsilon$, then if $x_{t}$ crosses the interval $J$ from right to left we must have a time $t'$ such that $x_t\in J$ and $t>n$. In such a case we have that $\{x_t\}$ does not converge to $0$ with positive probability. Thus, either the process crosses $x_t\geq 1-\epsilon$ and after that does not cross $J$ with positive probability or $x_t\in J$ and $t>n$. In any case $\{x_t\}$ does not converge to $0$ with positive probability. Hence, since the proportion $x_t$ does not converge to zero, this implies that $\limsup_{t\rightarrow\infty}\p_{F,b,\sigma}(a_t=1|\omega=0)>0$ and asymptotic learning fails.

\end{proof}
We next turn to the proof of Theorem \ref{th:gl}.
\begin{proof}[\textbf{Proof of Theorem \ref{th:gl}}]
Recall that the variance of $P_t$ is $v_t$ and $\inf_t v_t=v>0$. Hence there exists an information structure $F$ that is Blackwell-dominated $F_t$ for every $t$. This can be shown by constructing a binary information structure $F$ whose corresponding belief distribution function $P$ is a \emph{mean preserving contraction} of $P_t$ for every $t$. (That is, $P_t$ is a mean-preserving spread of $P$.) Such a binary information structure exists since $v_t\geq v$ for every $t$.

Consider the learning process both under $P_t$ and under $P$ conditional on state $\omega=1$. It follows from $\frac{dP_t^1}{dP_t}(y)=\frac{y}{\mu}$ (again see, e.g., Lemma 1 in \cite{arieli2020identifiable}) that conditional on $\omega=1$ and the observed proportion $x_t=x$, the probability of  choosing action $1$ equals
$$f_t(x)=\int \frac{2xy^2}{xy+(1-x)(1-y)}\mathrm{d}P_t(y).$$ And, similarly, for $P$ this probability equals $f(x)=\int \frac{2y^2}{xy+(1-x)(1-y)}\mathrm{d}P(x).$ Since $\frac{2y^2}{xy+(1-x)(1-y)}$ is convex in $y$ and since $P_t$ is a mean-preserving spread of $P$ it follows that $f_t(x)\geq f(x)$ for every $x$. It follows from Theorem \ref{th:main} that under $P$ the proportion of actions $1$ converges to $1$ as time goes to infinity. Therefore the coupling argument mentioned after the auxiliary lemmas implies that the proportion of actions $1$ converges to $1$ also with respect to the generalized processes and hence the probability $a_t=1$ converges to one conditional on state $\omega=1$. A similar derivation holds for $\omega=0$ with respect to $a_t=0$. Therefore asymptotic learning holds.

\end{proof}

\begin{proof}[\textbf{Proof of Theorem \ref{th:eyster}}]
 Assume by way of contradiction that $(e,\sigma)$ satisfies asymptotic learning for some information structure $F$. We first claim that $\lim_{x\gor 1}\sigma(x)=\sigma(1)=1$ (and similarly $\lim_{x\gor 0}\sigma(x)=\sigma(0)=0$). To see this, consider the case where the realized state is $\omega=1$. Since asymptotic learning holds, we must have that $\p(a_t=1|\omega=1)$ converges to one. Therefore, as time goes to infinity, the proportion of actions $1$ agent $t$ observes converges to one (in probability). Hence, if we let $m_t$ and $k_t$ be the number of $1$'s and $0$'s in the realized observation set of $t$, it follows that $e(k_t,m_t)$ approaches one. This implies that the posterior probability  $\frac{p(s_t)e(k_t,m_t)}{p(s)e(k_t,m_t)+(1-p(s_t))(1-e(k_t,m_t))}$ of agent $t$ also converges to one in probability.
  Since agent $t$ chooses action $1$ with probability $\int_S \sigma\Big(\frac{p(s)e(k_t,m_t)}{p(s)e(k_t,m_t)+(1-p(s))(1-e(k_t,m_t))}\Big)dF_1(s)=\p(a_t=1|\omega=1,e(k_t,m_t))$, which approaches one with $t$, we must have that $\lim_{x\gor 1}\sigma(x)=\sigma(1)=1$.  
 
 We note that since $F_0$ and $F_1$ are mutually absolutely continuous with respect to each other, and since $\lim_{x\gor 1}\sigma(x)=1$ we must have that $\p(a_t=1|e(k_t,m_t),\omega=0)=\\ \int_S \sigma\Big(\frac{p(s)e(k_t,m_t)}{p(s)e(k_t,m_t)+(1-p(s))(1-e(k_t,m_t))}\Big)dF_0(s)$ approaches one as $e(k_t,m_t)$ goes to one. This implies that for every value $c\in(0,1)$ there exists $b\in(0,1)$ such that if $e(k_t,m_t)>b$, then $\p(a_t=1|e(k_t,m_t),\omega=0)>c$.
 
 
 To conclude the proof, we couple our process of actions conditional on the realized state $\omega=0$ with a biased random walk such that if the random walk converges to infinity before ever crossing the origin, then the public belief of the learning process $e(k_t,m_t)$ converges to $1$. This will establish the desired contradiction. First consider a biased simple random walk with a success probability of $\frac{3}{4}$ starting at $1$. Thus at every stage the random walk position increases by $1$ with probability  $\frac{3}{4}$ and it decreases by $1$ with probability $\frac{1}{4}$. It follows using a simple recursive argument that the probability that the random walk converges to infinity before ever reaching $0$ is $\frac{2}{3}.$
 
 To apply the random walk argument to our action process, let $l=\Big\lceil\frac{\ln(\frac{r}{1-r})}{\ln(\frac{q}{1-q})}\Big\rceil$. Note that $l$ is a positive integer since $r<\frac{1}{2}$ and $l>\frac{1}{2}$. Choose $b$ such that $\p(a_t=1|e(k_t,m_t)\geq b,\omega=0)>(3/4)^{\frac{1}{l}}:=v.$  
 Take a time $t'$ such that $e(k_{t'},m_{t'})>b$. As long as $e(k_{t'},m_{t'})>b$, the probability that $a_t=1$ will be played for $l$ consecutive periods is, by assumption, at least $[(\frac{3}{4})^{\frac{1}{l}}]^l=\frac{3}{4}$. Moreover, the probability of choosing $a_t=0$ is less than $\frac 1 4$.  
 
 We now define a biased random walk $\{Z_n\}_{n\geq 0}$ and a corresponding increasing sequence of times $\{t(n)\}_{n\geq 0}$ as follows. We let  $t(0)=t'$ and let $Z_0=1$. If starting at time $t'$ the action $a_{t}=1$ was played for $l$ consecutive periods, i.e., $a_t=1$ for $t=t',\ldots,t'+l-1$, we let $Z_1=2$ and $t(1)=t'+l$. Otherwise, we let $Z_1=0$ and set $t(1)=\tilde t +1$, where  $\tilde t \in\{t',\ldots,t'+l-1\}$ is the first time $t$ such that $a_t=0$. Inductively, given that $Z_k$ is an integer and $t(k)$ is defined, we set $Z_{k+1}=Z_k+1$, and $t(k+1)=t(k)+l$ if $a_t=1$ was played in every time period $t=t(k),\ldots,t(k)+l-1$. Otherwise we let $Z_{k+1}=Z_k-1$, and $t(k+1)$ is defined by $\tilde t+1$ where $\tilde t\in\{t(k),\ldots,t(k)+l-1\}$ is the first time such that  $a_t=0$. 
 
 We call a sequence of realizations $\{z_n\}_{n\geq 0}$ of the random walk \emph{good}     
 if it does not reach zero. As mentioned above, the probability of a good sequence is at least $\frac{2}{3}$. We claim that if the sequence is good, then  $e(k_{t},m_{t})>b$ for every $t\geq t'$. Since $z_n=j>0$ for every $n$ it follows by definition that $k_t-k_{t'}\geq (m_t-m_{t'})l$ for every time $t$. We note that
 $e(k,m)=\frac{1}{1+(\frac{1-q}{q})^k(\frac{1-r}{r})^m}$. Therefore, for every $t\geq t'$ we can write
 \begin{align}
\notag&e(k_t,m_t)=\frac{1}{1+(\frac{1-q}{q})^{k_{t'}}(\frac{1-r}{r})^{m_{t'}}(\frac{1-q}{q})^{k_t-k_{t'}}(\frac{1-r}{r})^{m_t-m_{t'}}}\\
\label{eq:inp}&\geq  \frac{1}{1+(\frac{1-q}{q})^{k_{t'}}(\frac{1-r}{r})^{m_{t'}}(\frac{1-q}{q})^{({m_t-m_{t'}})l}(\frac{1-r}{r})^{m_t-m_{t'}}},
 \end{align}
 where inequality \eqref{eq:inp} follows since $\frac{(1-q)}{q}<1$ and $k_t-k_{t'}\geq (m_t-m_{t'})l$.
 Thus replacing $(\frac{1-q}{q})^{({k_t-k_{t'}})}$ with $(\frac{1-q}{q})^{({m_t-m_{t'}})l}$ increases the denominator and therefore decreases the expression.
 Taking log of  $(\frac{1-q}{q})^{({m_t-m_{t'}})l}(\frac{1-r}{r})^{m_t-m_{t'}}$ gives
 $$({m_t-m_{t'}})l\ln(\frac{1-q}{q})+({m_t-m_{t'}})\ln(\frac{1-r}{r}).$$
 Since $l=\Big\lceil\frac{\ln(\frac{r}{1-r})}{\ln(\frac{q}{1-q})}\Big\rceil\geq \frac{\ln(\frac{r}{1-r})}{\ln(\frac{q}{1-q})}$ and since $\ln(\frac{r}{1-r})$ is negative we have that
 $$({m_t-m_{t'}})l\ln(\frac{1-q}{q})+({m_t-m_{t'}})\ln(\frac{1-r}{r})\leq 0.$$
 Therefore $(\frac{1-q}{q})^{({m_t-m_{t'}})l}(\frac{1-r}{r})^{m_t-m_{t'}}\leq 1$. Integrating this inequality into \eqref{eq:inp} gives
 $e(k_t,m_t)\geq e(k_{t'},m_{t'})$, as desired.
 
 Thus, overall we have shown that if conditional on state $\omega=0$ the history reaches a time period $t'$ where $e(k_{t'},m_{t'})>b,$ then $e(k_{t},m_{t})>b,$ for every $t\geq t'$ with probability at least $\frac{2}{3}$. Note that, by the contradiction assumption, conditional on state $\omega=1$ the history reaches a time period $t'$ where $e(k_{t'},m_{t'})>b$ with probability one. Since $F_0$ is absolutely continuous with respect to $F_1$ and since existence of a time period $t'$ where $e(k_{t'},m_{t'})>b$ depends on finitely many signal realizations, we conclude that conditional on state $\omega=0$ there exists a positive probability of reaching a time period $t'$ such that $e(k_{t'},m_{t'})>b$. This is a contradiction since $\lim_{t\gor\infty}\p(a_t=0|\omega=0)=1$ only when $e(k_{t},m_{t})$ converges to zero in probability, conditional on state $\omega=0$. 
 

\end{proof}

\bibliographystyle{plain}
	\bibliography{herding}
	
\end{document}